\documentclass[10pt,journal,compsoc]{IEEEtran}
\usepackage{graphicx}
\usepackage{amsmath,amssymb,amsfonts,amsthm}
\usepackage{algorithm}
\usepackage{algorithmic}
\usepackage{booktabs,array,tabularx}
\usepackage{url}
\usepackage[hidelinks]{hyperref}
\hypersetup{pdftitle={Transport-Free Lower Bounds for Ollivier-Ricci Curvature in Higher-Order Networks: Scope, Tightness, and Computational Applications},pdfauthor={Wonwoo Kang and Heehyun Park}}
\newcommand{\supp}{\operatorname{supp}}
\newcommand{\R}{\mathbb{R}}
\newcommand{\one}{\mathbf{1}}
\newtheorem{theorem}{Theorem}
\newtheorem{proposition}[theorem]{Proposition}
\newtheorem{corollary}[theorem]{Corollary}

\newcolumntype{Y}{>{\raggedright\arraybackslash}X}
\newcolumntype{Z}[1]{>{\hsize=#1\hsize\linewidth=\hsize\raggedright\arraybackslash}X}
\newcommand{\runinheading}[1]{\par\noindent\textbf{#1}\enspace\ignorespaces}
\begin{document}
\title{Transport-Free Lower Bounds for Ollivier--Ricci Curvature in Higher-Order Networks: Scope, Tightness, and Computational Applications}
\author{Wonwoo Kang and Heehyun Park
\thanks{Both authors contributed equally and are listed in alphabetical order.}
\thanks{W. Kang is with the International Center for Mathematical Sciences--Sofia, Institute of Mathematics and Informatics, Bulgarian Academy of Sciences, Sofia, Bulgaria (e-mail: wonk@math.bas.bg).}
\thanks{H. Park is with the Department of Computer Science and Engineering, Pennsylvania State University, University Park, PA, USA (e-mail: hbp5148@psu.edu).}}
\IEEEtitleabstractindextext{
\begin{abstract}
We study how the Bauer--Jost--Liu curvature bound applies to hypergraph Ollivier--Ricci curvature (ORC) and how its accuracy depends on local mass and geometry. An explicit weighted-graph representation shows how the graph theorem applies. A transport construction transfers the bounds to monotone pairwise hyperedge aggregations. We analyze endpoint cancellation under laziness, quantify the exact improvement over the earlier mixture bound, and characterize the extremal geometry at fixed endpoint and overlap statistics. Further results quantify worst-case slack, stability, and exactness on isolated hyperedges. A capacity-constrained short-detour refinement uses additional geometric information and improves monotonically as the query budget increases. As a computational application, a threshold cascade preserves exact transport decisions. Numerical verification and public-hypergraph experiments examine tightness, walk dependence, and aggregation, including evaluations on 42,940 finite pairwise comparisons and 64,620 public-data instances. Together, the results specify the scope, equality mechanisms, and computational uses of BJL-type hypergraph bounds.
\end{abstract}
}
\maketitle
\IEEEdisplaynontitleabstractindextext
\IEEEpeerreviewmaketitle

\section{Introduction}
\label{sec:introduction}
Ollivier--Ricci curvature (ORC) measures whether local probability distributions are closer in Wasserstein distance than their centers are in the underlying metric~\cite{Ollivier2007,Ollivier2009}. It connects network geometry with optimal transport~\cite{Villani2003,Villani2009}, and has been used to study network organization, robustness, and bottlenecks~\cite{Samal2018,Arnaudon2021,ni2015ricci,sandhu2015graph,topping2021understanding}. The bounds of Jost and Liu and of Bauer, Jost, and Liu (BJL) relate graph curvature to endpoint mass and common-neighbor overlap~\cite{Jost2014,bauer2013ollivier}. Our objective is to carry this relation into the hypergraph setting and understand its scope, tightness, and dependence on the chosen walk.

Hypergraphs represent interactions involving several vertices at once. Their curvature depends on both the local probability measures and the rule assigning a curvature to an entire hyperedge. In the framework of Coupette, Dalleiger, and Rieck~\cite{Coupette2023}, Equal-Nodes (EN), Equal-Edges (EE), and Weighted-Edges (WE) walks encode different uses of incidence information, while mean and maximum aggregation combine pairwise transport distances. We ask which BJL-type inequalities these choices inherit, when equality holds, and what information is needed to improve a bound that uses only endpoint and overlap statistics.

The connection with the graph theorem is explicit. With the unit two-section metric, each of these three walks is a random walk on a suitable undirected weighted graph. Thus BJL already supplies their scalar pairwise bound, and coordinatewise monotonicity transfers any pairwise transport upper bound to the corresponding hyperedge lower bound. Section~\ref{sec:projection} records the representation. This transfer provides the hypergraph analysis; the further questions concern cancellation, equality, limits of the available statistics, and refinements using additional geometry.

Our earlier work~\cite{KangPark2024} presented the endpoint-and-overlap formula directly for neighborhood-supported measures and applied it to hypergraphs. Its lazy extension mixes a non-lazy coupling with transport between the endpoints, which can overlook mass cancellation. Here we compare that mixture with the BJL weighted-loop formula through a proof and an explicit saving identity. We then identify a geometry attaining the bound for every feasible statistic triple and every common laziness parameter, and analyze a short-detour refinement. The isolated hyperedge calculation connects these pairwise mechanisms to exact mean and maximum hyperedge curvature.

\runinheading{Contributions of the journal extension.}
The mathematical development combines (i) an explicit account of BJL applicability to the three hypergraph walks and monotone pairwise aggregation; (ii) an exact cancellation-saving identity, a quantified fixed-statistics extremal construction, a sharp worst-case additive gap, and a corrected sufficient cycle condition for local exactness; and (iii) an explicit shortcut capacity and saving formula, with a budget-monotone refinement. Stability and isolated hyperedge exactness make the effect of the input measures and laziness concrete. The results below specify the additional identities, constructions, and analyses rather than claiming a new scalar graph inequality.

As a computational application, we construct curvature-threshold subhypergraphs. A lower bound $L\geq\tau$ certifies membership; a value below $\tau$ leaves the decision unresolved. An exact-decision cascade refines unresolved bounds and solves transport when necessary. It therefore preserves the thresholded network in exact arithmetic. Small-graph checks, the retained eight-collection study, and new public-data comparisons examine tightness and computational cost. The residual-shell baseline is stronger on the tested workloads. Those comparisons delimit the practical role of the bounds alongside their mathematical analysis.

\runinheading{Relationship to earlier versions.}
The pairwise formula, lazy mixture bound, and hypergraph applications appeared in the 2024 preprint~\cite{KangPark2024}. The accepted conference paper~\cite{KangParkConference} clarified the support assumptions, formalized monotone aggregation, and supplied the independent LP experiments reused here. We do not claim those results as new journal contributions.

The development proceeds from the walk representation and pairwise transport construction to cancellation and geometric refinements, then to hyperedge aggregation and exactness. The central proofs accompany their statements. Screening and experiments follow as applications of this theory.

\section{Related Work}
\label{sec:related}
\runinheading{Transport-based network curvature.}
Ollivier's definition associates a local probability measure with each point of a metric space~\cite{Ollivier2007,Ollivier2009}. Choices of random walk and laziness change the resulting curvature~\cite{Lin2011}. Graph applications include Internet topology, biological networks, and multiscale network structure~\cite{Arnaudon2021,ni2015ricci,sandhu2015graph}. These applications motivate both structural curvature estimates and efficient evaluation.

\runinheading{Combinatorial bounds and local transport.}
Jost and Liu~\cite{Jost2014} and Bauer, Jost, and Liu~\cite{bauer2013ollivier} relate curvature to endpoint and common-neighbor information. BJL's weighted-graph theorem includes loops and already states sharpness, with transport constructions and Lipschitz equality witnesses in its proof. Our non-lazy formula $\Phi$ specializes to that bound without loops; the cancellation formula $\Psi^\alpha$ gives its common-laziness form, as identified in~\eqref{eq:bjl-identification}. The signed-flow presentation works directly with neighborhood-supported measures. The saving identity compares this formula with the earlier lazy mixture, and Proposition~\ref{prop:extremal} records an explicit fixed-statistics realization valid for all laziness parameters. These statements organize and analyze the BJL mechanism within the hypergraph framework. The support and unit-metric assumptions remain essential.

Gu, Zhang, and Sun's residual-shell bound (RS-LB)~\cite{Gu2026} allocates mass by distance and bounds the residual transport for weighted lazy $k$-step graph walks. Its theorem already permits feasible partial allocations. Our detour theorem is a specialization of the same general principle of combining a partial coupling with a residual bound: it retains the BJL endpoint routing and gives an explicit capacity under which removing shortcut mass $Q$ reduces that residual routing cost by $3Q$. The resulting saving identity and budget rule are stated in Section~\ref{sec:detours}. Note that they do not establish dominance over RS-LB. Both methods supply pairwise transport bounds that can be lifted to hyperedges. Section~\ref{sec:publicdata} compares the pinned authors' implementation directly.

\runinheading{Other curvature and transport constructions.}
Hypergraph ORC has been studied through multi-marginal transport, directed constructions, and unified walk-and-aggregation frameworks~\cite{Coupette2023,Asoodeh2018,Eidi2020ollivier,Leal2020}. We use pairwise aggregation and do not replace an optimized barycenter or multi-marginal objective by pairwise distances. Forman-type curvature~\cite{Forman2003,Leal2018} and the lower hypergraph curvature of Yang, Chen, and Li~\cite{Yang2025} provide alternative descriptors. Their interpretation and any comparison with ORC must respect the particular definitions. Entropic transport~\cite{Cuturi2013} offers another computational trade-off. Its regularized objective is not automatically a certified estimate of unregularized $W_1$. A feasible coupling evaluated with the original transport cost does supply an upper bound. Table~\ref{tab:related} summarizes these distinctions.

\begin{table*}[t]
\caption{Local transport bounds and their use in hypergraph ORC. Here $s_x,s_y$ are support sizes and $I$ is the number of Sinkhorn iterations. Costs exclude construction of measures and distances. Any valid pairwise $W_1$ upper bound can be lifted through a coordinate-wise nondecreasing aggregation.}
\label{tab:related}
\centering\footnotesize
\setlength{\tabcolsep}{4pt}
\renewcommand{\arraystretch}{1.10}
\begin{tabularx}{\textwidth}{@{}Z{0.88}Z{0.85}Z{1.04}Z{0.98}Z{1.25}@{}}
\toprule
\textbf{Method} & \textbf{ORC guarantee} & \textbf{Measures / locality} & \textbf{Local work} & \textbf{Limitation for this task}\\
\midrule
Transport LP & Exact optimum & General finite measures & $s_xs_y$ variables; solver-dependent & Repeated optimization and distance tables\\
\addlinespace[2.5pt]
Entropic OT~\cite{Cuturi2013} & Needs feasible-plan cost or error control & General finite measures & Dense $O(Is_xs_y)$ & Regularization and marginal residuals need control\\
\addlinespace[2.5pt]
Jost--Liu / weighted bounds~\cite{Jost2014,bauer2013ollivier} & One-sided in stated settings & Weighted graph walks, including loops & Local statistics & Applicability follows each theorem's assumptions\\
\addlinespace[2.5pt]
Earlier $U^\alpha$~\cite{KangParkConference} & One-sided & Arbitrary unit-neighborhood bases; common laziness & $O(s_x+s_y)$ with adjacency queries & Lazy cancellation and short detours can cause slack\\
\addlinespace[2.5pt]
RS-LB~\cite{Gu2026} & One-sided & Weighted lazy $k$-step graph walks & Shell construction and residual allocation & Requires more distance information than three statistics\\
\addlinespace[2.5pt]
BJL cancellation $\Psi^\alpha$ & One-sided; $\Psi^\alpha\leq U^\alpha$ & Same bases as $U^\alpha$ & $O(s_x+s_y)$; $O(1)$ after statistics & No new detour information; equals $\Phi$ at $\alpha=0$\\
\addlinespace[2.5pt]
Detour $\Gamma_b^\alpha$ (here) & One-sided; $\Gamma_b^\alpha\leq\Psi^\alpha$ & Exclusive-neighborhood routes of length 1 or 2 & Statistics plus $b$ route queries & Capacity and ordering limit the improvement; query costs vary\\
\addlinespace[2.5pt]
Forman-type~\cite{Forman2003,Leal2018} & Different curvature & Graphs / hypergraphs & Local incidence formula & Not an ORC certificate without a separate comparison theorem\\
\bottomrule
\end{tabularx}
\end{table*}

\section{Preliminaries and Hypergraph Walks}
\label{sec:setting}

\runinheading{Graphs.}
A simple undirected graph $G=(V,E_G)$ consists of a vertex set $V$ and an edge set $E_G\subseteq\binom{V}{2}$ of unordered pairs of distinct vertices. Vertices $x$ and $y$ are adjacent when $\{x,y\}\in E_G$. The open neighborhood is $N_G(x)=\{y:\{x,y\}\in E_G\}$, and the graph degree is $\deg_G(x)=|N_G(x)|$. With unit edge lengths, the shortest-path distance $d_G(x,y)$ is the minimum number of edges in a path from $x$ to $y$, defined within each connected component, with $d_G(x,x)=0$.

\runinheading{Metric and local measures.}
Let $(X,d)$ be a discrete metric space with $d(X\times X)\subseteq\mathbb{Z}_{\geq0}$. Write $x\sim y$ when $d(x,y)=1$, and define the open unit neighborhood
\[
N(x)=\{z\in X:d(x,z)=1\}.
\]
For a graph metric, this agrees with $N_G(x)$; we suppress the subscript when the underlying graph is clear. All local measures below have finite support. A \emph{non-lazy base measure} $\mu_x$ is a probability measure satisfying
\begin{equation}
\supp(\mu_x)\subseteq N(x).
\label{eq:support}
\end{equation}
This support assumption supplies the unit-length routes used below. An undirected weighted graph assigns symmetric weights $w_{xy}=w_{yx}>0$ to its edges and sets $w_{xy}=0$ otherwise, including $w_{xx}=0$ for the non-lazy base walk. At a vertex with $0<d_x<\infty$, the standard weighted walk is
\begin{equation}
\mu_x(y)=\frac{w_{xy}}{d_x}\one_{\{y\sim x\}},
\qquad d_x=\sum_{z\sim x}w_{xz}.
\label{eq:weighted-walk}
\end{equation}
Here $d_x$ is the weighted degree (row sum), whereas $\deg_G(x)$ counts neighbors. Weights determine transition probabilities. This integer-valued metric setting also includes integer lattices with $\ell^1$ distance and Hamming spaces when the chosen measures obey~\eqref{eq:support}.

For finitely supported probability measures $\mu$ and $\nu$, the 1-Wasserstein distance is
\begin{equation}
W_1(\mu,\nu)=\inf_{\pi\in\Gamma(\mu,\nu)}
\sum_{u,v\in X}d(u,v)\pi(u,v),
\label{eq:wasserstein}
\end{equation}
where $\Gamma(\mu,\nu)$ is the set of couplings with marginals $\mu$ and $\nu$. ORC is
\begin{equation}
\kappa(x,y)=1-\frac{W_1(\mu_x,\mu_y)}{d(x,y)}.
\label{eq:orc}
\end{equation}
For the adjacent pairs considered below, $\kappa(x,y)=1-W_1(\mu_x,\mu_y)$.

\subsection{Hypergraph measures and metric}
\label{sec:hypergraph-walks}
A finite undirected hypergraph is a pair $H=(V,E)$, where $V$ is a finite vertex set and $E$ is a finite family of nonempty subsets of $V$, called hyperedges. A vertex $x$ is incident to a hyperedge $e$ when $x\in e$, and $|e|$ is its size. We allow repeated subsets as distinct indexed hyperedges; all sums and incidence counts include their multiplicity. A finite simple graph is the special case with distinct hyperedges of size two. The hypergraph incidence degree is $\deg(x)=|\{e\in E:x\in e\}|$.

The \emph{two-section graph} $G_H=(V,E_2)$ is the simple graph with
\[
E_2=\bigl\{\{x,y\}\in\tbinom{V}{2}:\exists e\in E,\ \{x,y\}\subseteq e\bigr\}.
\]
Thus distinct vertices are adjacent precisely when they share a hyperedge. We use its unit shortest-path metric and write $N(x)=N_{G_H}(x)$. In general, the incidence degree $\deg(x)$ differs from the graph degree $\deg_{G_H}(x)=|N(x)|$. We evaluate pairs within a connected component and hyperedges of size at least two. Singleton hyperedges create no two-section adjacency, and vertices with no nonsingleton incident hyperedge have no base walk and are excluded.

We recall the three local measures of~\cite{Coupette2023}. The Equal-Nodes walk assigns equal mass to each vertex in $N(x)$:
\[
\mu_x^{\mathrm{EN}}(y)=\frac{1}{|N(x)|},\qquad y\in N(x).
\]
The Equal-Edges walk first selects a nonsingleton incident hyperedge uniformly and then another vertex within it:
\[
\mu_x^{\mathrm{EE}}(y)=
\frac{1}{\deg(x)-|\{e\ni x:|e|=1\}|}
\sum_{e\supseteq\{x,y\}}\frac{1}{|e|-1}.
\]
The Weighted-Edges walk weights the first choice by the number of alternative vertices:
\[
\mu_x^{\mathrm{WE}}(y)=
\sum_{e\supseteq\{x,y\}}
\frac{|e|-1}{\sum_{f\ni x}(|f|-1)}\frac{1}{|e|-1}.
\]
All three base measures are supported on $N(x)$. The barycentric aggregator instead optimizes a measure $\bar\mu$ and is not covered by Theorem~\ref{thm:monotone}:
\[
\begin{aligned}
\operatorname{AGG}_{\mathrm B}(e)&=\frac{1}{|e|-1}\sum_{i\in e}W_1(\mu_i,\bar\mu),\\
\bar\mu&\in\arg\min_{\nu}\sum_{i\in e}W_1(\mu_i,\nu).
\end{aligned}
\]

\subsection{Weighted-graph representation and scope}
\label{sec:projection}
For $x\ne y$, define symmetric weights
\begin{equation}
\begin{aligned}
w^{\mathrm{EN}}_{xy}&=\one_{\{y\in N(x)\}},\\
w^{\mathrm{EE}}_{xy}&=\sum_{e\supseteq\{x,y\}}\frac{1}{|e|-1},\\
w^{\mathrm{WE}}_{xy}&=|\{e\in E:\{x,y\}\subseteq e\}|.
\end{aligned}
\label{eq:hypergraph-weights}
\end{equation}
Normalizing each row gives exactly its named walk, with hyperedges counted according to the multiplicity convention above. Indeed, the row sums are $|N(x)|$ for EN, the number of nonsingleton incident hyperedges for EE, and $\sum_{e\ni x}(|e|-1)$ for WE. Every positive-weight edge has unit metric length, so these graph representations preserve both the transition probabilities and the two-section distance. The EN and WE clique-expansion interpretations are discussed in~\cite{Coupette2023}. The EE identity follows directly by summing its incidence formula.

For $0\leq\alpha<1$, adding a loop of weight $\alpha d_x/(1-\alpha)$ at each vertex $x$ realizes $\mu_x^\alpha=\alpha\delta_x+(1-\alpha)\mu_x$; the loop contributes once to the transition row sum and does not change the metric. The endpoint $\alpha=1$ is immediate. BJL therefore applies directly to all three walks. This representation identifies the inherited pairwise inequality. Hyperedge grouping still determines which pairs are aggregated, and incidence structure determines the EE/WE weights. An unweighted two-section alone need not determine those walks. The following measure-level proofs make the bounds explicit without assuming a global reversible walk for arbitrary input measures.

\section{Bauer--Jost--Liu Bounds in Measure Form}
\label{sec:pairwise}
The non-lazy formula below is the BJL weighted-graph bound written in terms of local probability masses, retained from~\cite{KangPark2024,KangParkConference}. Its transport proof also applies to any measures satisfying~\eqref{eq:support}.

\subsection{Overlap Statistics and the Bound}

Fix $x\sim y$ and put $C(x,y)=N(x)\cap N(y)$. Define
\begin{align}
p_{xy}&=\mu_x(y)+\mu_y(x),\nonumber\\
c_{\wedge}(x,y)&=\sum_{z\in C(x,y)}\min\{\mu_x(z),\mu_y(z)\},\nonumber\\
c_{\vee}(x,y)&=\sum_{z\in C(x,y)}\max\{\mu_x(z),\mu_y(z)\},
\label{eq:statistics}
\end{align}
and, with $(a)_+=\max\{a,0\}$,
\begin{equation}
\Phi(\mu_x,\mu_y)=1+(1-p_{xy}-c_{\vee})_+
 +(1-p_{xy}-c_{\wedge})_+-c_{\wedge}.
\label{eq:phi}
\end{equation}
The statistic $c_{\wedge}$ is mass that can be matched at common neighbors at zero cost. The difference $c_{\vee}-c_{\wedge}$ records imbalance on that common support. Endpoint mass $p_{xy}$ can be routed across the central edge at unit cost, while all remaining mass is connected by a path $u,x,y,v$ of length at most three.

\runinheading{Example.}
On a triangle $x,y,z$, let $\mu_x(y)=1/3$, $\mu_x(z)=2/3$, and $\mu_y(x)=\mu_y(z)=1/2$. Then $p_{xy}=5/6$, $c_{\wedge}=1/2$, and $c_{\vee}=2/3$, so $\Phi=1/2$. Half the mass is matched at $z$ at zero cost; the remaining mass can be moved across unit edges. Thus the bound is exact in this example.

\begin{theorem}[Pairwise bound]
\label{thm:pairwise}
Let $(X,d)$ satisfy Section~\ref{sec:setting}, let $x\sim y$, and let $\mu_x,\mu_y$ be probability measures supported on $N(x)$ and $N(y)$. Then
\begin{equation}
W_1(\mu_x,\mu_y)\leq\Phi(\mu_x,\mu_y),
\qquad
\kappa(x,y)\geq1-\Phi(\mu_x,\mu_y).
\label{eq:pairwise-bound}
\end{equation}
\end{theorem}

\begin{proof}
Write $a=\mu_x(y)$, $b=\mu_y(x)$, $m=c_{\wedge}$, $M=c_{\vee}$, and $p=a+b$. Match the common mass at each vertex of $C=C(x,y)$ at zero cost. Put
\begin{align*}
q_x&=\sum_{z\in C}\mu_x(z),&q_y&=\sum_{z\in C}\mu_y(z),\\
r_x&=1-a-q_x,&r_y&=1-b-q_y,\\
c_+&=q_x-m,&c_-&=q_y-m.
\end{align*}
The exclusive source mass $r_x$ can be routed to $x$ and the exclusive demand $r_y$ can be supplied from $y$, at total cost $r_x+r_y$. Each remaining common-neighbor surplus or deficit can use either endpoint, since its distance to both is one. Routing these masses to or from the selected endpoints costs $c_++c_-=M-m$.

Before assigning the common-neighbor residuals, the signed surplus at $x$ is $r_x-b$. Assigning any fraction of each common residual to either endpoint allows every value in
\[
[r_x-b-c_-,\ r_x-b+c_+]=[1-p-M,\ 1-p-m]=[A,B].
\]
The surplus at $y$ is its negative. Choose the value in $[A,B]$ nearest zero and balance the endpoints along $xy$. This costs $A_++(-B)_+$. All routes have nonnegative flows with divergence $\mu_x-\mu_y$. Cancel directed cycles and decompose the remaining finite flow into source-to-target paths. Adding the matched mass gives a coupling whose metric cost is no greater than the route cost. Consequently,
\begin{align*}
W_1&\leq r_x+r_y+c_++c_-+A_++(-B)_+\\
&=2-p-2m+A_++(-B)_+\\
&=1-m+A_++B_+=\Phi.
\end{align*}
The last identity uses $B+(-B)_+=B_+$. Subtracting from one proves the curvature bound.
\end{proof}

\subsection{Three transport regimes}

The compact positive-part notation hides a useful case structure. Let
\[
A=1-p_{xy}-c_{\vee},\qquad B=1-p_{xy}-c_{\wedge};
\]
then $A\leq B$. Direct expansion gives
\begin{equation}
\Phi=
\begin{cases}
3-2p_{xy}-c_{\vee}-2c_{\wedge},&0\leq A\leq B,\\
2-p_{xy}-2c_{\wedge},&A<0\leq B,\\
1-c_{\wedge},&A\leq B<0.
\end{cases}
\label{eq:regimes}
\end{equation}
In the first expansion, some mass may use the three-step route $u,x,y,v$. In the second, endpoint mass removes the three-step residual, but a two-step imbalance remains. In the third, endpoint mass also absorbs that imbalance, so only zero-cost common mass improves upon the unit baseline. The bound uses the largest possible distance within each cost class and does not inspect squares or other short detours.

The common-neighbor bound is no larger than the endpoint-only bound.

\begin{proposition}[Endpoint-only comparison]
\label{prop:dominance}
Under Theorem~\ref{thm:pairwise},
\begin{equation}
\Phi(\mu_x,\mu_y)
\leq1+2(1-p_{xy})_+.
\label{eq:endpoint-dominance}
\end{equation}
Consequently,
\[
\kappa(x,y)\geq1-\Phi(\mu_x,\mu_y)
\geq-2(1-p_{xy})_+.
\]
\end{proposition}

\begin{proof}
Because $c_{\vee}\geq c_{\wedge}\geq0$, each positive-part term in~\eqref{eq:phi} is at most $(1-p_{xy})_+$. Subtracting $c_{\wedge}$ can only decrease the result.
\end{proof}

Thus the overlap terms can tighten the endpoint-only bound. The inequality is strict when the overlap removes positive transport cost beyond a boundary case.

\subsection{Equality Cases and Stability}

The bound is exact when the worst-case three-step routes are forced by the local geometry.

\begin{proposition}[Equality without short detours]
\label{prop:tree-exact}
Suppose the assumptions of Theorem~\ref{thm:pairwise} hold, $C(x,y)=\varnothing$, and
\[
d(u,v)=3
\quad\text{for every }u\in N(x)\setminus\{y\},
\quad v\in N(y)\setminus\{x\}.
\]
Then
\begin{equation}
W_1(\mu_x,\mu_y)=\Phi(\mu_x,\mu_y)
=1+2(1-\mu_x(y)-\mu_y(x))_+.
\label{eq:tree-exact}
\end{equation}
\end{proposition}

\begin{proof}
Let $p=\mu_x(y)+\mu_y(x)$. The supports are disjoint, so all mass costs at least one. If $p<1$, the target mass on $N(y)\setminus\{x\}$ can receive at most $\mu_x(y)$ from $y$ at unit cost. At least $1-p$ mass must therefore travel from $N(x)\setminus\{y\}$ to $N(y)\setminus\{x\}$ at cost three, adding $2(1-p)$ above the unit baseline. If $p\geq1$, the unit baseline is attainable. These lower bounds equal the feasible upper bound in Theorem~\ref{thm:pairwise}, proving~\eqref{eq:tree-exact}.
\end{proof}

For an unweighted graph edge lying on no cycle of length three, four, or five, the uniform walk therefore gives
\begin{equation}
\kappa(x,y)=-2\left(1-\frac{1}{\deg_G(x)}-\frac{1}{\deg_G(y)}\right)_+.
\label{eq:tree-degree}
\end{equation}
The cycle condition is sufficient, not necessary. Excluding triangles and squares alone is insufficient: on a five-cycle, the two exclusive neighbors can be at distance two, violating the distance-three hypothesis. Short detours can therefore reduce transport cost even when there is no triangle or square.

\begin{proposition}[Sharp worst-case non-lazy slack]
\label{prop:worstcase}
Under Theorem~\ref{thm:pairwise},
\begin{equation}
0\leq\Phi-W_1(\mu_x,\mu_y)\leq A_++B_+\leq2.
\label{eq:worstcase}
\end{equation}
The constant two is attained within the allowed class of measures.
\end{proposition}
\begin{proof}
Since distinct points have distance at least one, transport costs at least the total variation distance. The only common support is in $C$, so $W_1\geq1-m$. Subtract this from $\Phi=1-m+A_++B_+$; also $A,B\leq1$. On the four-cycle $x,y,v,u,x$, choose $\mu_x=\delta_u$ and $\mu_y=\delta_v$. Then $p=m=M=0$, $\Phi=3$, and $W_1=d(u,v)=1$.
\end{proof}
The attaining example allows zero transition probabilities on some graph edges. Strictly positive edge-weight walks can approach the same gap by giving the central edge arbitrarily small weight. Thus a useful certificate need not be an accurate point estimate.

The formula is also stable with respect to perturbations of the three sufficient statistics. This matters when weighted transition probabilities are estimated or rounded.

\begin{proposition}[Statistic-level stability]
\label{prop:stability}
Let $\Phi(p,c_{\wedge},c_{\vee})$ denote~\eqref{eq:phi}. For any two feasible statistic triples,
\begin{align}
|\Phi(p,c_{\wedge},c_{\vee})-\Phi(p',c'_{\wedge},c'_{\vee})|
\leq{}&2|p-p'|+2|c_{\wedge}-c'_{\wedge}|\nonumber\\
&+|c_{\vee}-c'_{\vee}|.
\label{eq:stability}
\end{align}
\end{proposition}

\begin{proof}
The map $t\mapsto(t)_+$ is 1-Lipschitz. Apply the triangle inequality to the two positive-part terms and the final $-c_{\wedge}$ term in~\eqref{eq:phi}.
\end{proof}

\subsection{A common laziness parameter}

For $\alpha\in[0,1]$, define the standard lazy local measure
\begin{equation}
\mu_x^{\alpha}=\alpha\delta_x+(1-\alpha)\mu_x.
\label{eq:lazy-measure}
\end{equation}
The same $\alpha$ is used at every point, as in the usual $\alpha$-lazy random walk~\cite{Lin2011}. Vertex-dependent self-masses are outside the present common-parameter result.

\begin{theorem}[Lazy bound]
\label{thm:lazy}
For adjacent $x\sim y$ and non-lazy base measures $\mu_x,\mu_y$,
\begin{equation}
W_1(\mu_x^{\alpha},\mu_y^{\alpha})
\leq U_{xy}^{\alpha}:=\alpha+(1-\alpha)\Phi(\mu_x,\mu_y),
\label{eq:lazy-w}
\end{equation}
and hence
\begin{equation}
\kappa^{\alpha}(x,y)
\geq L_{xy}^{\alpha}:=(1-\alpha)(1-\Phi(\mu_x,\mu_y)).
\label{eq:lazy-kappa}
\end{equation}
\end{theorem}

\begin{proof}
Let $\pi$ be an optimal coupling of $\mu_x$ and $\mu_y$. Then
\[
\pi^\alpha=\alpha\delta_{(x,y)}+(1-\alpha)\pi
\]
couples $\mu_x^\alpha$ and $\mu_y^\alpha$. Its transport cost is
$\alpha d(x,y)+(1-\alpha)W_1(\mu_x,\mu_y)$. Because $x\sim y$, Theorem~\ref{thm:pairwise} gives
\[
W_1(\mu_x^\alpha,\mu_y^\alpha)
\leq\alpha+(1-\alpha)W_1(\mu_x,\mu_y)
\leq\alpha+(1-\alpha)\Phi(\mu_x,\mu_y).
\]
The curvature statement follows from~\eqref{eq:orc}.
\end{proof}

This mixture construction is always feasible but need not be optimal because lazy endpoint mass may cancel non-lazy endpoint imbalance. Accordingly, non-lazy exactness does not automatically imply exactness for $0<\alpha<1$.

Since $0\leq\Phi\leq3$, the pairwise lower bound lies in
\begin{equation}
-2(1-\alpha)\leq L_{xy}^{\alpha}\leq1-\alpha.
\label{eq:range}
\end{equation}
If $L_{xy}^{\alpha}\geq\tau$, then $\kappa^{\alpha}(x,y)\geq\tau$. The converse need not hold.

\subsection{Sparse evaluation and complexity}
\label{sec:algorithm}
The statistics depend on the geometric common neighborhood, not just the intersection of the nonzero supports. In particular, a common neighbor with positive mass on only one side contributes to $c_{\vee}$. Algorithm~\ref{alg:pairwise} scans the support union and includes a vertex only if it belongs to both unit neighborhoods. It counts each vertex once.

\begin{algorithm}[t]
\caption{Statistics for pairwise certificates}
\label{alg:pairwise}
\begin{algorithmic}[1]
\STATE \textbf{procedure} \textsc{PairStatistics}$(x,y,\mu_x,\mu_y)$
\STATE $p\gets\mu_x(y)+\mu_y(x)$; $m\gets0$; $M\gets0$
\FORALL{distinct $z\in\supp(\mu_x)\cup\supp(\mu_y)$}
  \IF{$d(x,z)=1$ and $d(y,z)=1$}
    \STATE $m\gets m+\min\{\mu_x(z),\mu_y(z)\}$
    \STATE $M\gets M+\max\{\mu_x(z),\mu_y(z)\}$
  \ENDIF
\ENDFOR
\STATE \textbf{return} $(p,m,M)$

\end{algorithmic}
\end{algorithm}

Let $s_x=|\supp(\mu_x)|$ and $s_y=|\supp(\mu_y)|$. With sparse mass maps and constant-time unit-adjacency queries, the statistics take expected time
\begin{equation}
O(s_x+s_y).
\label{eq:pair-complexity}
\end{equation}
One can scan all keys of $\mu_x$ and then keys of $\mu_y$ absent from $\mu_x$, using constant auxiliary storage beyond the input maps. Creating a materialized union uses $O(s_x+s_y)$ storage. Sorted lists offer a deterministic merge, with the cost of geometric membership queries added separately. Constructing the walks, incidence indexes, or two-section adjacency is preprocessing and is not included in~\eqref{eq:pair-complexity}. This distinction matters when sparse supports omit geometric neighbors.

\section{Endpoint Cancellation for Lazy Walks}
\label{sec:refinement}
The mixture in Theorem~\ref{thm:lazy} is feasible, but it need not cancel opposing endpoint flows. The following measure-level statement recovers the weighted-loop bound of~\cite[Theorem~2]{bauer2013ollivier} when applied to weighted graph walks, and makes its improvement over the mixture explicit. We use the same signed routing construction directly on the lazy measures. Set $t=1-\alpha$, and abbreviate $p=p_{xy}$, $m=c_{\wedge}$, $M=c_{\vee}$, $A=1-p-M$, and $B=1-p-m$.

\begin{theorem}[Closed-form cancellation refinement]
\label{thm:refinement}
Under Theorem~\ref{thm:pairwise}, for every $\alpha\in[0,1]$ define
\begin{equation}
\begin{split}
\Psi_{xy}^{\alpha}={}&t(2-p-2m)\\
 &+(\alpha+tA)_++(-\alpha-tB)_+.
\end{split}
\label{eq:psi}
\end{equation}
Then
\begin{equation}
W_1(\mu_x^\alpha,\mu_y^\alpha)\leq\Psi_{xy}^{\alpha}\leq U_{xy}^{\alpha}.
\label{eq:psi-dominance}
\end{equation}
More precisely,
\begin{equation}
\begin{split}
U_{xy}^{\alpha}-\Psi_{xy}^{\alpha}
={}&\min\{\alpha,t(p+M-1)_+\}\\
 &+\min\{\alpha,t(p+m-1)_+\}.
\end{split}
\label{eq:psi-saving}
\end{equation}
For $0<\alpha<1$, the improvement is strict if and only if $p+M>1$. At $\alpha=0$, $\Psi_{xy}^{0}=\Phi$, and at $\alpha=1$, $\Psi_{xy}^{1}=1$.
\end{theorem}

\begin{proof}
Use $r_x,r_y,c_+,c_-$ in the proof of Theorem~\ref{thm:pairwise}. Cancel identical mass at every location. Equivalently, work with the signed difference of the two lazy measures. The exclusive and common-neighbor residuals are the previous residuals scaled by $t$. Routing them to or from the endpoints costs
\[
t(r_x+r_y+c_++c_-)=t(2-p-2m).
\]
Relative to the non-lazy balance, laziness adds signed mass $\alpha$ at $x$ and $-\alpha$ at $y$. Thus the possible signed surplus at $x$, after selecting routes for the common residuals, is exactly the interval
\[
[\alpha+tA,\ \alpha+tB].
\]
Choose its point nearest zero and route the remaining imbalance across $xy$. Since $A\leq B$, this costs $(\alpha+tA)_++(-\alpha-tB)_+$. Flow decomposition, followed by restoring common mass, gives a feasible coupling of the lazy measures with cost at most~\eqref{eq:psi}. This proves the first inequality.

For dominance, use $2-p-2m=1-m+B$ to rewrite~\eqref{eq:psi} as
\[
\Psi_{xy}^{\alpha}=t(1-m)-\alpha+(\alpha+tA)_++(\alpha+tB)_+.
\]
For every real $z$ and nonnegative $t,\alpha$,
\[
\alpha+t z_+-(\alpha+t z)_+=\min\{\alpha,t(-z)_+\}.
\]
Apply this identity with $z=A,B$ and subtract from $U_{xy}^{\alpha}=\alpha+t(1-m+A_++B_+)$. This gives~\eqref{eq:psi-saving}. Its terms are nonnegative; when $0<\alpha<1$, at least one is positive exactly when $A<0$, because $A\leq B$. The endpoint cases follow directly.
\end{proof}

The improved curvature certificate is
\begin{equation}
\widehat L_{xy}^{\alpha}=1-\Psi_{xy}^{\alpha}\geq L_{xy}^{\alpha}.
\label{eq:refined-curvature}
\end{equation}
Evaluating $\Psi$ from $(p,m,M)$ takes constant time. Its total local cost is therefore the same as~\eqref{eq:pair-complexity}. The first term in~\eqref{eq:psi} is nonnegative because $2-p-2m=r_x+r_y+c_++c_-$. The refined bound satisfies $0\leq\Psi\leq U\leq3-2\alpha$, hence
\[
-2(1-\alpha)\leq\widehat L_{xy}^{\alpha}\leq1.
\]
The upper endpoint differs from the earlier interval~\eqref{eq:range}: cancellation can certify curvature one, including at an intermediate laziness parameter.

For the graph representations in Section~\ref{sec:projection}, BJL's Theorem~2 becomes
\begin{equation}
\kappa^\alpha(x,y)\geq tm+2\alpha
 -(\alpha+tA)_+-(\alpha+tB)_+
 =1-\Psi_{xy}^\alpha.
\label{eq:bjl-identification}
\end{equation}
The two loop probabilities contribute $2\alpha$, while the base endpoint and overlap masses scale by $t$. Thus~\eqref{eq:psi-saving} quantifies the advantage of using the weighted-loop formula in place of the earlier mixture; it is not a stronger scalar inequality than BJL for these walks.

\subsection{Extremal geometry at fixed statistics}
BJL's proof already exhibits equality through Lipschitz potentials when short routes are absent~\cite[Theorem~2]{bauer2013ollivier}. The next proposition makes the quantifiers explicit: for each feasible triple, one graph and one pair of base measures realize equality simultaneously for all common laziness parameters.
\begin{proposition}[A sharp certificate from three statistics]
\label{prop:extremal}
For every feasible triple $(p,m,M)$ arising in Theorem~\ref{thm:pairwise}, there is a finite unweighted graph and a pair of neighborhood-supported measures with those same statistics for which
\[
W_1(\mu_x^\alpha,\mu_y^\alpha)=\Psi_{xy}^{\alpha}
\quad\text{for every }\alpha\in[0,1].
\]
Consequently, any upper bound valid for all such metric spaces and measures, and depending only on $(p,m,M,\alpha)$, is at least $\Psi_{xy}^{\alpha}$ at every feasible input.
\end{proposition}
\begin{proof}
Retain $x,y$, their supported exclusive neighbors, and the common neighbors with positive mass on either side. Build a graph using only $xy$, the edge from each exclusive neighbor to its own endpoint, and the two edges from each common neighbor to $x$ and $y$. Keep the original measures. This graph preserves $p,m,M$.

Theorem~\ref{thm:refinement} supplies the upper bound. For the reverse inequality, let $h\in[-1,1]$, put $f(x)=h$ and $f(y)=0$, and assign value $h+1$ to each exclusive source neighbor and $-1$ to each exclusive target neighbor. Give each common neighbor with positive signed residual value $\min(h,0)+1$, and each with negative signed residual value $\max(h,0)-1$; vertices with zero residual can receive any feasible value between these extremes. The resulting function is 1-Lipschitz on every edge, hence for the graph metric.

With the notation of the routing proof, its integral against $\mu_x^\alpha-\mu_y^\alpha$ is
\[
t(2-p-2m)+
\begin{cases}
h(\alpha+tA),&0\leq h\leq1,\\
h(\alpha+tB),&-1\leq h\leq0.
\end{cases}
\]
Maximizing over $h$ yields $\Psi_{xy}^{\alpha}$, since $A\leq B$. For every coupling, the integral of a 1-Lipschitz function against the marginal difference is at most the transport cost. Thus $W_1\geq\Psi$, proving equality. Applying any universally valid statistic-only upper bound to this graph proves the final assertion.
\end{proof}

This extremal statement quantifies the limitation of the input statistics over the full class of admissible neighborhood measures. A uniformly tighter certificate on that class must inspect additional geometry or impose further restrictions. In particular, the proposition does not assert extremality within each prescribed EN, EE, or WE family, or on a fixed hypergraph. Those restrictions can supply information absent from $(p,m,M)$.

\subsection{A budgeted short-detour refinement}
\label{sec:detours}
Let $X_{xy}=N(x)\setminus(C(x,y)\cup\{y\})$ and $Y_{xy}=N(y)\setminus(C(x,y)\cup\{x\})$ be the exclusive neighborhoods, and put $h=(\alpha+tA)_+$. For inspected pairs $(u,v)\in X_{xy}\times Y_{xy}$ with verified $d(u,v)\leq2$, assign nonnegative shortcut masses $q_{uv}$ satisfying
\begin{equation}
\begin{aligned}
\sum_v q_{uv}&\leq t\mu_x(u),&
\sum_u q_{uv}&\leq t\mu_y(v),\\
Q:=\sum_{u,v}q_{uv}&\leq h.
\end{aligned}
\label{eq:detour-capacity}
\end{equation}
The last constraint limits the mass taken from the three-step part of the hub routing. Omitting it would not be justified by the following argument.

\begin{theorem}[Feasible shortcut refinement]
\label{thm:detour}
For every allocation satisfying~\eqref{eq:detour-capacity},
\begin{equation}
W_1(\mu_x^\alpha,\mu_y^\alpha)
\leq\Gamma^\alpha(q):=\Psi_{xy}^\alpha-
\sum_{u,v}(3-d(u,v))q_{uv}
\leq\Psi_{xy}^\alpha.
\label{eq:detour-bound}
\end{equation}
The second inequality is strict whenever $Q>0$.
\end{theorem}
\begin{proof}
The claim is immediate when $Q=0$. Otherwise $\alpha+tA>0$ and $Q\leq\alpha+tA$. Remove the assigned source and target masses from the two lazy measures; both remaining measures have mass $1-Q$. In the signed-flow construction of Theorem~\ref{thm:refinement}, the exclusive source and target totals each decrease by $Q$, while the common-neighbor residuals and endpoint signed masses stay fixed. The cost of routing these residuals to the endpoints therefore decreases by $2Q$, and the possible surplus interval at $x$ becomes
\[
[\alpha+tA-Q,\ \alpha+tB-Q].
\]
Its endpoints are nonnegative, so the central-edge cost decreases by $Q$. The remaining subprobability measures consequently admit a coupling of cost at most $\Psi_{xy}^\alpha-3Q$, by the same flow-decomposition argument. Add the direct shortcut coupling with entries $q_{uv}$, whose cost is $\sum d(u,v)q_{uv}$. The combined coupling has the original marginals and gives~\eqref{eq:detour-bound}. Since every shortcut has length at most two, its saving is positive when its assigned mass is positive.
\end{proof}

\runinheading{An explicit budget rule.}
Fix an ordering of the supported pairs in $X_{xy}\times Y_{xy}$. Inspect at most $b$ candidates, counting unsuccessful distance queries toward the budget. On finding a distance-one or distance-two route, transfer the minimum of the unused source mass, unused target mass, and unused capacity $h-Q$. Retain earlier allocations as the budget increases. The resulting certificates satisfy
\begin{equation}
W_1\leq\Gamma_{b+1}^\alpha\leq\Gamma_b^\alpha\leq\Gamma_0^\alpha=\Psi_{xy}^\alpha.
\label{eq:detour-hierarchy}
\end{equation}
No transport optimization is required. The result can depend on candidate order; the guarantee does not. Our implementation uses lexicographic vertex order. The budget limits inspected pairs, not the numerical amount of transported mass.

After forming the statistics and exclusive supports, the additional cost is $O(b+\sum_{r=1}^{b}c_r)$, where $c_r$ is the cost of the $r$th route query. A distance-one test uses adjacency membership; a distance-two test may scan the smaller of two adjacency sets, costing $O(\min\{\deg_G(u),\deg_G(v)\})$. Distance discovery is therefore accounted for separately. The process can stop when $1-\Gamma_b^\alpha\geq\tau$, or when the budget or a mass capacity is exhausted.

For the four-cycle $x,y,v,u,x$ with non-lazy measures $\mu_x=\delta_u$ and $\mu_y=\delta_v$, the value is $\Psi^0=3$. One shortcut of mass one and length one gives $\Gamma_1^0=1=W_1$. On the five-cycle $x,y,v,z,u,x$, the same measures give $\Gamma_1^0=2=W_1$ using a length-two shortcut. These examples improve the bound using information beyond the three statistics; they do not imply exactness on arbitrary square-rich or pentagon-rich networks. The refinement reserves a feasible partial coupling and reuses the endpoint routing for the remaining subprobability measures. General partial allocation is also available in the residual-shell framework~\cite{Gu2026}. The explicit capacity and saving here describe this particular residual bound. Budget monotonicity follows from retaining feasible allocations as further routes are inspected.

\section{Hypergraph Bounds, Walk Dependence, and Exactness}
\label{sec:hypergraph}

Let $H=(V,E)$ be a hypergraph with the unit two-section metric of Section~\ref{sec:hypergraph-walks}. Thus distinct $i,j$ satisfy $i\sim j$ whenever some hyperedge contains both. The Equal-Nodes, Equal-Edges, and Weighted-Edges walks of~\cite{Coupette2023} are all supported on $N(i)$; their formulas are given in that section. The weighted representations in~\eqref{eq:hypergraph-weights} identify their BJL bounds. Theorems~\ref{thm:lazy}, \ref{thm:refinement}, and~\ref{thm:detour} give the corresponding mixture, cancellation, and detour estimates for every pair in a hyperedge.

The passage from pairwise bounds to hyperedge bounds is the monotonicity argument retained from the conference version. It applies beyond the average and maximum. For a hyperedge $e$, enumerate its $m_e=\binom{|e|}{2}$ unordered pairs and collect the distances and their upper bounds in vectors
\[
\mathbf{W}_e=(W_1(\mu_i^\alpha,\mu_j^\alpha))_{\{i,j\}\subseteq e},
\qquad
\mathbf{U}_e=(V_{ij})_{\{i,j\}\subseteq e},
\]

where $V_{ij}$ is any valid upper bound, such as $U_{ij}^\alpha$, $\Psi_{ij}^\alpha$, $\Gamma_{b,ij}^\alpha$, a residual-shell bound, or an exact transport value.

\begin{theorem}[Monotone pairwise aggregation]
\label{thm:monotone}
Let $F_e:\R_{\geq0}^{m_e}\to\R_{\geq0}$ be coordinatewise nondecreasing and define
\[
\kappa_{F}(e)=1-F_e(\mathbf{W}_e).
\]
Then
\begin{equation}
\kappa_F(e)\geq1-F_e(\mathbf{U}_e).
\label{eq:monotone-lift}
\end{equation}
\end{theorem}

\begin{proof}
The pairwise guarantees give $\mathbf{W}_e\leq\mathbf{U}_e$ coordinatewise. Monotonicity gives $F_e(\mathbf{W}_e)\leq F_e(\mathbf{U}_e)$; subtracting both sides from one reverses the inequality and proves~\eqref{eq:monotone-lift}.
\end{proof}

This theorem covers nonnegative weighted averages, the maximum, and fixed nondecreasing norms or order statistics of pairwise distances. Barycentric aggregation is excluded. Its optimized barycenter is not a coordinatewise function of the precomputed vertex-pair distances and need not be a neighborhood measure anchored at a vertex.

For the average and maximum aggregators of~\cite{Coupette2023},
\begin{align}
\operatorname{AGG}_{\mathrm A}(e)
&=\frac{2}{|e|(|e|-1)}\sum_{\{i,j\}\subseteq e}
W_1(\mu_i^\alpha,\mu_j^\alpha),\nonumber\\
\operatorname{AGG}_{\mathrm M}(e)
&=\max_{\{i,j\}\subseteq e}W_1(\mu_i^\alpha,\mu_j^\alpha),
\label{eq:aggregators}
\end{align}
Theorem~\ref{thm:monotone} immediately yields
\begin{align}
\kappa_{\mathrm A}(e)
&\geq1-\frac{2}{|e|(|e|-1)}\sum_{\{i,j\}\subseteq e}V_{ij},
\label{eq:average-bound}\\
\kappa_{\mathrm M}(e)
&\geq1-\max_{\{i,j\}\subseteq e}V_{ij}.
\label{eq:maximum-bound}
\end{align}
For $V_{ij}=U_{ij}^\alpha$, the right-hand sides lie in~\eqref{eq:range}; using $V_{ij}=\Psi_{ij}^\alpha$ permits the sharper upper endpoint one. Nonnegative averages of these pairwise or hyperedge lower bounds also give node-level lower bounds.

If $s_i=|\supp(\mu_i)|$, computing either bound for one hyperedge costs
\begin{equation}
O\!\left(\sum_{\{i,j\}\subseteq e}(s_i+s_j)\right)=O\!\left((|e|-1)\sum_{i\in e}s_i\right).
\label{eq:hyperedge-complexity}
\end{equation}
Pairs shared by several hyperedges can be cached. Explicit evaluation of all constituent pairs is quadratic in hyperedge cardinality. This is a cost of the present implementation, not an information theoretic lower bound for every possible screening algorithm.

\subsection{Walk dependence from hyperedge incidence}
Consider $E=\{\{x,y,z\},\{x,w\}\}$ and the pair $x,y$. The EE walk assigns masses $(1/4,1/4,1/2)$ to $(y,z,w)$ from $x$, and masses $(1/2,1/2)$ to $(x,z)$ from $y$. Hence
\[
(p,m,M)=(3/4,1/4,1/2),\qquad \Phi=3/4.
\]
For EN and WE, the walk from $x$ is uniform on its three neighbors and the walk from $y$ is unchanged, giving
\[
(p,m,M)=(5/6,1/3,1/2),\qquad \Phi=2/3.
\]
Both bounds are exact: match the common mass at $z$, move the mass at $w$ to $x$, and distribute the remaining mass at $y$ to the deficits at $x,z$, all over unit distances. Thus the non-lazy pair curvature is $1/4$ for EE and $1/3$ for EN/WE. Replacing the triple hyperedge by its three constituent two-vertex edges preserves the unweighted two-section but makes EE uniform as well. This example isolates how incidence choices enter the BJL statistics. The weighted representations retain precisely the information needed for each pairwise transport problem.

\subsection{Isolated Hyperedges}

The next result separates the non-lazy value from the gap introduced by the lazy coupling.

\begin{proposition}[Isolated hyperedge]
\label{prop:isolated}
Let $e$ have size $k\geq2$, suppose its vertices have no neighbors outside $e$, and use the Equal-Nodes base walk. For distinct $i,j\in e$,
\begin{equation}
W_1(\mu_i,\mu_j)=\Phi(\mu_i,\mu_j)=\frac{1}{k-1}.
\label{eq:isolated-nonlazy}
\end{equation}
Thus both non-lazy bounds~\eqref{eq:average-bound}--\eqref{eq:maximum-bound} are exact and
\[
\kappa_{\mathrm A}(e)=\kappa_{\mathrm M}(e)=\frac{k-2}{k-1}.
\]
For a common laziness $\alpha$,
\begin{equation}
W_1(\mu_i^\alpha,\mu_j^\alpha)=\frac{|\alpha k-1|}{k-1},
\qquad
U_{ij}^\alpha=\alpha+\frac{1-\alpha}{k-1}.
\label{eq:isolated-lazy}
\end{equation}
The lazy upper-bound gap is
\begin{equation}
U_{ij}^\alpha-W_1(\mu_i^\alpha,\mu_j^\alpha)
=\begin{cases}
2\alpha,&0\leq\alpha\leq1/k,\\
2(1-\alpha)/(k-1),&1/k\leq\alpha\leq1.
\end{cases}
\label{eq:isolated-gap}
\end{equation}
\end{proposition}

\begin{proof}
The base measures agree on the $k-2$ vertices other than $i,j$ and move the remaining mass $1/(k-1)$ over unit distance. Under laziness, they still agree on those $k-2$ vertices and differ only at $i$ and $j$ by $|\alpha-(1-\alpha)/(k-1)|$. Moving that imbalance across the unit edge gives the exact distance. Substitution into~\eqref{eq:lazy-w} gives $U_{ij}^\alpha$, and elementary simplification yields~\eqref{eq:isolated-gap}.
\end{proof}

\begin{corollary}[Exact refined certificate on isolated hyperedges]
\label{cor:isolated-refined}
Under Proposition~\ref{prop:isolated},
\[
\Psi_{ij}^{\alpha}=W_1(\mu_i^\alpha,\mu_j^\alpha)=\frac{|\alpha k-1|}{k-1}
\]
for every $\alpha\in[0,1]$. Thus refined mean and maximum hyperedge bounds are exact for every laziness parameter.
\end{corollary}
\begin{proof}
Here $p=2/(k-1)$ and $m=M=(k-2)/(k-1)$, so $2-p-2m=0$ and $A=B=-1/(k-1)$. Equation~\eqref{eq:psi} reduces to $|\alpha-(1-\alpha)/(k-1)|$, which is the exact value in~\eqref{eq:isolated-lazy}.
\end{proof}
For example, with $k=4$ and $\alpha=1/4$, all four lazy measures are identical. The exact transport distance and $\Psi$ are zero, whereas $U=1/2$. At a curvature threshold $\tau=3/4$, the old bound is unresolved and the refinement certifies both pairwise and hyperedge decisions without an LP. This is an analytic example, not a measured network-scale speedup.

\section{Application: Certified Screening and Complexity}
\label{sec:cascade}
\subsection{Threshold decisions}
Let $\mathcal Q$ be a finite set of pair or hyperedge queries, with local measures fixed on the original network. For a pair set $K_q=1-W_q$; for a hyperedge set $K_q=1-F_q(\mathbf W_q)$, with $F_q$ fixed and coordinatewise nondecreasing. The desired decision is $D_q=\one_{\{K_q\geq\tau\}}$. Initialize the distance bounds with $\Psi^\alpha$, apply the budgeted detour refinement when unresolved, and use exact pairwise transport only if needed. The original $U^\alpha$ remains an ablation baseline; it is not a separate production stage because both scalar formulas share the same statistics.

\begin{algorithm}[t]
\caption{Certified threshold screening}
\label{alg:cascade}
\begin{algorithmic}[1]
\REQUIRE Queries $\mathcal Q$, threshold $\tau$, fixed local measures
\FORALL{$q\in\mathcal Q$}
\STATE Obtain cached statistics for the constituent pairs
\STATE Form $\mathbf V_q$ using $\Psi^\alpha$
\IF{$1-F_q(\mathbf V_q)\geq\tau$}
  \STATE $D_q\gets1$; \textbf{continue}
\ENDIF
\STATE Replace unresolved entries of $\mathbf V_q$ by $\Gamma_b^\alpha$
\IF{$1-F_q(\mathbf V_q)\geq\tau$}
  \STATE $D_q\gets1$; \textbf{continue}
\ENDIF
\STATE Solve missing pairwise transports; cache exact values
\STATE $D_q\gets\one_{\{1-F_q(\mathbf W_q)\geq\tau\}}$
\ENDFOR
\STATE \textbf{return} $(D_q)_{q\in\mathcal Q}$
\end{algorithmic}
\end{algorithm}
For a single pair, $F_q$ is the identity. A previously cached exact value may always replace an upper bound. Optional certified geometric refinements may be inserted by taking the coordinatewise minimum with the current bound.

\begin{theorem}[Exact decisions and safe early acceptance]\label{thm:cascade}
Algorithm~\ref{alg:cascade} returns the same decisions as solving all constituent transport problems. Every early acceptance is correct. The set accepted by the refined stage contains the set accepted by the original bound.
\end{theorem}
\begin{proof}
At every stage $\mathbf W_q\leq\mathbf V_q$ coordinatewise. Monotonicity gives $K_q\geq1-F_q(\mathbf V_q)$, so the early-acceptance test implies $K_q\geq\tau$. Every query not accepted early is evaluated exactly. The hierarchy $\Gamma_b^\alpha\leq\Psi^\alpha\leq U^\alpha$ can only increase the curvature lower bound, proving the final assertion.
\end{proof}

For maximum aggregation, certification requires every pairwise upper bound to be at most $1-\tau$. For a weighted average with fixed $w_{ij}\geq0$ and $\sum w_{ij}=1$, the requirement is $\sum w_{ij}V_{ij}\leq1-\tau$. Weights selected adaptively from unknown exact curvatures do not satisfy this fixed-aggregation assumption. During fallback, one may recompute the aggregate after each exact pair solve and accept early; the theorem still applies.

\runinheading{Network interpretation.}
Define the curvature-threshold subhypergraph by retaining precisely the hyperedges with $K_e\geq\tau$. The cascade constructs the same subhypergraph as all-pairs exact transport. This is screening with measures fixed on the original network. If deleting edges is followed by recomputing walks or curvature, the bounds must be recomputed at every iteration.

\subsection{Cost and certification coverage}
Let $\mathcal P$ be the set of distinct queried vertex pairs and $\mathcal R\subseteq\mathcal P$ the pairs actually solved by the LP fallback. Write $T_{\rm pre}$ for loading and constructing measures and indexes, $B_{ij}$ for the certificate cost, $D_{ij}$ for a required distance table, $P_{ij}$ for the LP solve, and $T_{\rm agg}$ for query aggregation and bookkeeping. With a shared cache, the accounted work is
\begin{equation}
\begin{split}
T_{\rm cascade}={}&T_{\rm pre}+\sum_{\{i,j\}\in\mathcal P} B_{ij}+T_{\rm agg}\\
&+\sum_{\{i,j\}\in\mathcal R}(D_{ij}+P_{ij}).
\end{split}
\label{eq:cascade-time}
\end{equation}
The corresponding all-LP workload replaces $\mathcal R$ by $\mathcal P$ and omits the certificate pass. The term $B_{ij}$ includes any detour queries and their actual discovery cost. Shared distance preprocessing must be charged once, rather than once per pair; the implementation below includes it in the measured total. A speedup requires the avoided distance/LP work to exceed the certificate and additional bookkeeping costs. An extra tier may reduce the solve count without reducing elapsed time.

Define analytic certification coverage as the fraction of queries accepted using only analytic bounds. Separately report queries requiring no new LP calls, specifying the query order and whether cached exact values were used. Distinguish both from pair solve avoidance, $1-|\mathcal R|/|\mathcal P|$, because overlapping hyperedges share pairwise work. Under equal per-pair solve cost $c$ and certificate cost $b$, the solve-related terms reduce to $|\mathcal P|[b+(1-\eta)c]$, where $\eta$ is pair solve avoidance. Without equal costs,~\eqref{eq:cascade-time}, rather than a single unresolved fraction, is the relevant comparison.

Caching scalar pair results requires $O(|\mathcal P|)$ space. If local LPs are processed sequentially, only one $O(s_i s_j)$ distance table need be active, in addition to solver workspace and stored measures. For incidence input, explicit two-section construction can inspect $\sum_{e\in E}\binom{|e|}{2}$ candidate pairs before deduplication. Neither preprocessing nor total memory is unconditionally linear in the incidence count.

\subsection{Numerical interpretation}
Theorems~\ref{thm:refinement}, \ref{thm:detour}, and~\ref{thm:cascade} are exact statements. Floating-point evaluation must preserve the one-sided inequality if a computational certificate is claimed. If $\epsilon_{ij}$ bounds the evaluation error, use $\widetilde V_{ij}+\epsilon_{ij}$ as the upper bound, or use outward-rounded interval arithmetic. A fixed tolerance by itself is a testing convention, not a proof of such an error bound. When the threshold lies inside a numerical uncertainty interval, increase precision or use verified primal/dual bounds. The experiments below check agreement with numerical exact-OT references, not universal machine-arithmetic certification.

\section{Validation and Public-Network Experiments}
\label{sec:evaluation}
We examine validity, tightness, walk dependence, and threshold-screening cost. The Supplementary Material provides the detailed experimental protocols and reported results (Section~S1), node summaries (S2), dataset and solver details (S3), and additional figures (S4). All numerical comparisons use the documented $10^{-9}$ convention; they do not establish machine-arithmetic certification.

\subsection{Finite verification}
\label{sec:newchecks}
The new checks enumerate all 771 connected labeled simple graphs on two to five vertices, using uniform and seeded nonuniform neighborhood measures at five laziness values. Across 42,940 pairwise comparisons, no violation of $W_1\leq\Psi\leq U$ exceeds $10^{-9}$, and the saving identity agrees within $4.44\times10^{-16}$. All 36,780 aggregate and 318,880 threshold-decision checks pass under the stated tolerance rule. A separate extremal test verifies 400 equality instances. At detour budgets 4, 16, and 64, all 128,820 comparisons satisfy the bound and budget monotonicity; budget 64 improves 5,757 cases, including 816 non-lazy cases. These are finite checks, exhaustive over the stated graph topologies but not all probability measures. Supplement~S1.1 gives the measures, seeds, boundary conventions, and executable checks.

\subsection{Public-hypergraph tightness}
\label{sec:publicdata}
We use the ORCHID archive~\cite{Coupette2023}: all 37 SHA hypergraphs, 12 uniformly sampled MUS hypergraphs, and 512 sampled NDC-PC adjacent pairs. EN, EE, and WE walks at $\alpha\in\{0,1/4,1/2\}$ give 64,620 pair--walk--laziness instances. These repeated measurements are dependent. Independent SciPy/HiGHS LPs supply reference transport values. The RS-LB comparison uses the authors' pinned compiled implementation, shell depth three, and one thread; Supplement~S1.2 records all settings and sampling details.

\begin{table*}[t]
\centering\footnotesize
\caption{New public-data transport slack, pooling the three walks and three laziness values within each dataset. Entries are empirical median / 90th percentile, not confidence intervals. The final column is the percentage with $\Psi-\Gamma_{64}>10^{-9}$. MUS uses 12 sampled hypergraphs; NDC-PC uses 512 sampled pairs.}
\label{tab:new-tightness}
\begin{tabular*}{\textwidth}{@{\extracolsep{\fill}}lrrrrrr@{}}
\toprule
Dataset & Instances & $U-W_1$ & $\Psi-W_1$ & $\Gamma_{64}-W_1$ & RS-LB slack & Strict gain (\%)\\
\midrule
SHA & 44,667 & 0.199 / 0.400 & 0.112 / 0.363 & 0.085 / 0.332 & 0.000 / 0.067 & 23.8\\
MUS & 15,345 & 0.310 / 0.531 & 0.253 / 0.525 & 0.183 / 0.454 & 0.022 / 0.096 & 49.9\\
NDC-PC & 4,608 & 0.395 / 0.719 & 0.281 / 0.708 & 0.229 / 0.672 & 0.010 / 0.074 & 40.5\\
\bottomrule
\end{tabular*}
\end{table*}

No bound or hierarchy violation exceeds $10^{-9}$. Table~\ref{tab:new-tightness} shows strict detour gains in 23.8\%, 49.9\%, and 40.5\% of the respective instances. RS-LB has substantially smaller median and upper-tail slack in these samples; no pointwise or universal ordering is claimed. Supplement~S1.2 also reports intermediate budgets and plots tightness against laziness and certification coverage against threshold, separating approximation quality from screening utility.

\subsection{Complete screening workloads}
\label{sec:fullpipeline}
The complete workloads use all SHA hypergraphs, MUS graph 412, and all of NDC-PC. Each method answers pair, mean-hyperedge, and maximum-hyperedge queries in that order with a shared pair cache, fixed EN walks, $\alpha=1/2$, and $\tau=0$. The fallback and all-exact baseline use the one-thread unregularized POT network-simplex solver~\cite{Flamary2021}; a 1,637-case cross-check against HiGHS differs by at most $9.44\times10^{-15}$. Timings include loading, walk and distance construction, bounds, fallback, aggregation, and retained components, excluding interpreter/import startup and result serialization. They are medians of three sequential fresh-process runs on a shared host. Full ranges, memory, preprocessing, and ablations appear in Supplement~S1.3.

\begin{table*}[t]
\centering\footnotesize
\caption{Complete screening workloads: median elapsed seconds over three runs and percentage of unique pair transports never solved. All methods reproduce the numerical exact-OT decisions. SHA covers all 37 hypergraphs; MUS is graph 412; NDC-PC is the full hypergraph. Full results, including $U$, $\Psi$, timing ranges, and peak RSS, are in Supplement~S1.3.}
\label{tab:pipeline-summary}
\begin{tabular*}{\textwidth}{@{\extracolsep{\fill}}lrrrrr@{}}
\toprule
& \multicolumn{3}{c}{Median time (s)} & \multicolumn{2}{c}{Solve avoidance (\%)}\\
Workload & Exact OT & $\Gamma_{64}$ & RS-LB & $\Gamma_{64}$ & RS-LB\\
\midrule
SHA & 0.617 & 0.288 & 0.157 & 72.7 & 90.3\\
MUS & 1.153 & 0.973 & 0.905 & 96.4 & 100.0\\
NDC-PC & 7.738 & 6.871 & 4.167 & 29.5 & 63.2\\
\bottomrule
\end{tabular*}
\end{table*}

The detour cascade gives median speedups of 2.14$\times$, 1.19$\times$, and 1.13$\times$ over all-exact transport (Table~\ref{tab:pipeline-summary}). RS-LB avoids more solves and is faster on every workload. On MUS, detours avoid more solves than $\Psi$ but take longer, illustrating refinement overhead. Peak RSS is similar across methods; no memory advantage is demonstrated. All 50 pipeline runs, including repetitions and a non-lazy SHA check, match the corresponding numerical exact-OT decisions.

\subsection{Retained studies and network output}
\label{sec:casestudy}
Supplement~S1.4 reports threshold subhypergraphs of SHA graph 12. At $\tau=0$, mean aggregation retains 107 hyperedges over 47 vertices in one component; maximum transport retains 80 over 41 vertices in two components. This illustrates aggregation-dependent structure without claiming a scientifically optimal threshold or improved community detection.

The inherited studies are retained in Supplement~S1.5--S1.7: 1,199,780 finite local-measure checks, the eight-collection 512-pair LP comparison, 192 controlled-geometry cases, 160 hyperedge comparisons, and 20 lazy isolated-hyperedge cases. Their kernel timings use preconstructed inputs and do not estimate full-pipeline speedup. The original figures and complete table remain available, including the large synthetic gaps. Because $\Psi^0=U^0=\Phi$, scalar cancellation alone does not alter those non-lazy results; the new detour tier was not part of that study.

\section{Discussion and Limitations}
\label{sec:limitations}
The hypergraph objective has a precise boundary. For the EN/EE/WE walks and unit two-section metric, weighted-graph representation and monotone aggregation transfer an existing BJL inequality. The journal development adds explicit comparisons and geometric analysis around that foundation: the cancellation-saving identity, fixed-statistics realization, worst-case slack, and shortcut capacity, together with hypergraph exactness and evaluation. The general routing and Lipschitz-duality mechanisms have precedents in BJL, and partial coupling precedes the present detour specialization. These connections determine the appropriate scope of the contribution claims.

The extremal result identifies what can be achieved from endpoint and overlap statistics alone. The detour tier goes beyond that information and can improve non-lazy cases, but its capacity constraint and candidate ordering limit its effectiveness. A fixed budget gives predictable query counts, not constant-time distance discovery or a guarantee of near-optimal transport.

The public-data comparison favors RS-LB in tightness, and the complete workloads show that the tighter certificate often pays for itself. The mathematical role of the hierarchy is to identify which mass statistics and additional routes account for the bound and its improvement. Its screening use is workload dependent; the measurements do not establish it as the fastest available implementation. The large inherited non-lazy synthetic gaps have not been reevaluated with the new detour tier, so they cannot be claimed to have been resolved.

Our complete workloads reach 1,251 vertices and 65,436 indexed nonsingleton hyperedges. They establish feasibility on those inputs, not large-vertex-network scalability. The benchmark uses a dense cutoff-three distance matrix when required, incurring quadratic storage; a local or sparse distance service is necessary before extrapolating to much larger vertex sets. Explicit two-section expansion can also dominate preprocessing. The three heterogeneous workloads are not a controlled asymptotic scaling study.

There are further empirical limits. Tightness is sampled for MUS and NDC-PC; repeated pairs across walks and laziness values are dependent. The experiments provide empirical quantiles and repeat-run timing ranges, not population confidence intervals. Full-pipeline timings use EN at one principal threshold and laziness, with a non-lazy SHA check; broader walk/threshold timing and alternative candidate orderings require additional runs. A feasible-plan entropic-OT baseline has not been benchmarked. The analytic kernels are Python and RS-LB is compiled C++, and the shared host limits fine runtime comparisons. Numerical validation does not replace outward-rounded certification or proof review.

Finally, threshold-subhypergraph construction is an exact-decision application, not evidence of improved community detection, anomaly identification, or biological interpretation. Establishing such a benefit requires a task-specific dataset, evaluation criterion, and an independently motivated threshold. The current case study reports retained structure without claiming those additional outcomes.

\section{Conclusion}
We have developed BJL-type lower bounds for hypergraph ORC with explicit attention to the walk, metric, and aggregation rule. Weighted-graph representation identifies the inherited scalar inequality for EN, EE, and WE, and monotonicity gives hyperedge bounds. The signed-flow analysis quantifies cancellation under laziness, realizes extremality at fixed statistics, and explains the additional information used by a capacity-constrained detour refinement. Stability, worst-case slack, and isolated-hyperedge exactness further describe the bounds' scope.

The retained and expanded experiments connect these mechanisms to hypergraph data. Exact fallback supplies a threshold-screening application, while the RS-LB comparison records the practical limits of the present refinements. Further mathematical questions concern equality and sharper bounds under prescribed hypergraph incidence and walk constraints, which are more restrictive than the arbitrary neighborhood measures in the extremal proposition. Larger sparse implementations and application-specific validation would extend the computational evidence.

\section*{Data and Code Availability}
The code, selected public inputs, and numerical outputs for the new finite checks and public-hypergraph experiments are available in the reproducibility package at \url{https://doi.org/10.5281/zenodo.23801938}. The package provides reproduction instructions, dependency versions, input provenance, and verification reports. The original raw records and plotting scripts for the inherited studies are not included; their supplied figures are retained in the package.

\section*{Acknowledgments}
This work is supported by the Ministry of Education and Science of the Republic of Bulgaria, grant DO1-239/10.12.2024, and the Simons Foundation, grant SFI-MPS-T-Institutes-00007697.

\section*{AI Use Statement}
This manuscript builds on our 2024 preprint~\cite{KangPark2024}, whose underlying research was conducted without assistance from LLMs. An AI assistant was used to help prepare and edit the present manuscript. All mathematical results and proofs are the authors' own; the authors reviewed the AI-assisted text and take full responsibility for the manuscript's content.


\begin{thebibliography}{10}
\providecommand{\url}[1]{#1}
\csname url@samestyle\endcsname
\providecommand{\newblock}{\relax}
\providecommand{\bibinfo}[2]{#2}
\providecommand{\BIBentrySTDinterwordspacing}{\spaceskip=0pt\relax}
\providecommand{\BIBentryALTinterwordstretchfactor}{4}
\providecommand{\BIBentryALTinterwordspacing}{\spaceskip=\fontdimen2\font plus
\BIBentryALTinterwordstretchfactor\fontdimen3\font minus
  \fontdimen4\font\relax}
\providecommand{\BIBforeignlanguage}[2]{{
\expandafter\ifx\csname l@#1\endcsname\relax
\typeout{** WARNING: IEEEtran.bst: No hyphenation pattern has been}
\typeout{** loaded for the language `#1'. Using the pattern for}
\typeout{** the default language instead.}
\else
\language=\csname l@#1\endcsname
\fi
#2}}
\providecommand{\BIBdecl}{\relax}
\BIBdecl

\bibitem{Ollivier2007}
Y.~Ollivier, ``{Ricci} curvature of metric spaces,'' \emph{Comptes Rendus
  Math{\'e}matique}, vol. 345, no.~11, pp. 643--646, 2007.

\bibitem{Ollivier2009}
------, ``{Ricci} curvature of {Markov} chains on metric spaces,''
  \emph{Journal of Functional Analysis}, vol. 256, no.~3, pp. 810--864, 2009.

\bibitem{Villani2003}
C.~Villani, \emph{Topics in Optimal Transportation}, ser. Graduate Studies in
  Mathematics.\hskip 1em plus 0.5em minus 0.4em\relax Providence, RI: American
  Mathematical Society, 2003, vol.~58.

\bibitem{Villani2009}
------, \emph{Optimal Transport: Old and New}, ser. Grundlehren der
  Mathematischen Wissenschaften.\hskip 1em plus 0.5em minus 0.4em\relax Berlin:
  Springer, 2009, vol. 338.

\bibitem{Samal2018}
A.~Samal, R.~P. Sreejith, J.~Gu, S.~Liu, E.~Saucan, and J.~Jost, ``Comparative
  analysis of two discretizations of {Ricci} curvature for complex networks,''
  \emph{Scientific Reports}, vol.~8, no.~1, p. 8650, 2018.

\bibitem{Arnaudon2021}
A.~Gosztolai and A.~Arnaudon, ``Unfolding the multiscale structure of networks
  with dynamical {Ollivier--Ricci} curvature,'' \emph{Nature Communications},
  vol.~12, no.~1, p. 4561, 2021.

\bibitem{ni2015ricci}
C.-C. Ni, Y.-Y. Lin, J.~Gao, X.~D. Gu, and E.~Saucan, ``{Ricci} curvature of
  the {Internet} topology,'' in \emph{2015 IEEE Conference on Computer
  Communications (INFOCOM)}.\hskip 1em plus 0.5em minus 0.4em\relax IEEE, 2015,
  pp. 2758--2766.

\bibitem{sandhu2015graph}
R.~Sandhu, T.~Georgiou, E.~Reznik, L.~Zhu, I.~Kolesov, Y.~Senbabaoglu, and
  A.~Tannenbaum, ``Graph curvature for differentiating cancer networks,''
  \emph{Scientific Reports}, vol.~5, no.~1, p. 12323, 2015.

\bibitem{topping2021understanding}
J.~Topping, F.~Di~Giovanni, B.~P. Chamberlain, X.~Dong, and M.~M. Bronstein,
  ``Understanding over-squashing and bottlenecks on graphs via curvature,'' in
  \emph{International Conference on Learning Representations}, 2022.

\bibitem{Coupette2023}
C.~Coupette, S.~Dalleiger, and B.~Rieck, ``{Ollivier--Ricci} curvature for
  hypergraphs: A unified framework,'' in \emph{The Eleventh International
  Conference on Learning Representations}.\hskip 1em plus 0.5em minus
  0.4em\relax OpenReview.net, 2023.

\bibitem{KangPark2024}
\BIBentryALTinterwordspacing
W.~Kang and H.~Park, ``Accelerated evaluation of {Ollivier--Ricci} curvature
  lower bounds: Bridging theory and computation,'' arXiv:2405.13302v1, 2024.
  [Online]. Available: \url{https://arxiv.org/abs/2405.13302v1}
\BIBentrySTDinterwordspacing

\bibitem{KangParkConference}
------, ``Transport-free lower bounds for {Ollivier--Ricci} curvature in
  higher-order networks,'' accepted for publication in \emph{Complex Networks
  \& Their Applications XV: Proceedings of the Fifteenth International
  Conference on Complex Networks and Their Applications: COMPLEX NETWORKS
  2026}, forthcoming. Author manuscript dated October 3, 2026; DOI not yet
  available.

\bibitem{bauer2013ollivier}
F.~Bauer, J.~Jost, and S.~Liu, ``{Ollivier--Ricci} curvature and the spectrum
  of the normalized graph {Laplace} operator,'' \emph{Mathematical Research
  Letters}, vol.~19, no.~6, pp. 1185--1205, 2012.

\bibitem{Lin2011}
Y.~Lin, L.~Lu, and S.-T. Yau, ``{Ricci} curvature of graphs,'' \emph{Tohoku
  Mathematical Journal}, vol.~63, no.~4, pp. 605--627, 2011.

\bibitem{Jost2014}
J.~Jost and S.~Liu, ``Ollivier's {Ricci} curvature, local clustering and
  curvature-dimension inequalities on graphs,'' \emph{Discrete \& Computational
  Geometry}, vol.~51, no.~2, pp. 300--322, 2014.

\bibitem{Gu2026}
X.~Gu, H.~Zhang, and J.~Sun, ``A residual-shell-based lower bound for
  {Ollivier--Ricci} curvature,'' \emph{Applied Mathematics Letters}, vol. 183,
  2026, article 110095.

\bibitem{Asoodeh2018}
S.~Asoodeh, T.~Gao, and J.~Evans, ``Curvature of hypergraphs via multi-marginal
  optimal transport,'' in \emph{2018 IEEE Conference on Decision and Control
  (CDC)}, 2018, pp. 1180--1185.

\bibitem{Eidi2020ollivier}
M.~Eidi and J.~Jost, ``{Ollivier--Ricci} curvature of directed hypergraphs,''
  \emph{Scientific Reports}, vol.~10, no.~1, p. 12466, 2020.

\bibitem{Leal2020}
W.~Leal, M.~Eidi, and J.~Jost, ``{Ricci} curvature of random and empirical
  directed hypernetworks,'' \emph{Applied Network Science}, vol.~5, no.~1,
  p.~65, 2020.

\bibitem{Forman2003}
R.~Forman, ``Bochner's method for cell complexes and combinatorial {Ricci}
  curvature,'' \emph{Discrete \& Computational Geometry}, vol.~29, no.~3, pp.
  323--374, 2003.

\bibitem{Leal2018}
W.~Leal, G.~Restrepo, P.~F. Stadler, and J.~Jost, ``Forman--{Ricci} curvature
  for hypergraphs,'' \emph{Advances in Complex Systems}, vol.~24, no.~1, pp.
  Paper No. 2\,150\,003, 24, 2021.

\bibitem{Yang2025}
S.~Yang, C.~Chen, and D.~Li, ``Lower {Ricci} curvature for hypergraphs,'' 2025.

\bibitem{Cuturi2013}
M.~Cuturi, ``Sinkhorn distances: Lightspeed computation of optimal transport,''
  in \emph{Advances in Neural Information Processing Systems}, vol.~26, 2013.

\bibitem{Flamary2021}
R.~Flamary \emph{et~al.}, ``POT: Python Optimal Transport,'' \emph{Journal
of Machine Learning Research}, vol.~22, no.~78, pp.~1--8, 2021.

\end{thebibliography}
\end{document}